\documentclass[11pt]{article}

\usepackage[margin=1in]{geometry}
\usepackage{amsmath,amsfonts,amssymb,amsthm,mathtools,bm}
\usepackage{enumitem}
\usepackage{graphicx}
\usepackage{subcaption}
\usepackage{booktabs}
\usepackage{microtype}
\usepackage{hyperref}

\usepackage{tikz}
\usetikzlibrary{arrows.meta,positioning}

\newtheorem{theorem}{Theorem}
\newtheorem{lemma}{Lemma}
\newtheorem{proposition}{Proposition}

\theoremstyle{definition}
\newtheorem{definition}{Definition}
\theoremstyle{remark}
\newtheorem{remark}{Remark}

\newcommand{\R}{\mathbb{R}}

\newcommand{\ip}[2]{\left\langle #1,\, #2 \right\rangle}

\newcommand{\Regret}{\operatorname{Regret}}
\newcommand{\cM}{\mathcal{M}}

\newcommand{\argmin}{\operatorname{argmin}}
\newcommand{\rank}{\operatorname{rank}}

\newcommand{\figdir}{figs/}

\newcommand{\Fig}[2]{\includegraphics[width=#2\linewidth]{\figdir#1}}

\title{\bf Predictable Gradient Manifolds in Deep Learning:\\
Temporal Path-Length and Intrinsic Rank as a Complexity Regime}

\author{
Anherutowa Calvo\\
\texttt{ac1180@princeton.edu}
}

\date{}

\begin{document}
\maketitle

% ============================================================
\begin{abstract}
Worst-case analyses of first-order optimization treat gradient sequences as
adversarial or noisy objects whose complexity scales with the horizon $T$ and
ambient dimension $d$. In modern deep learning, however, observed gradient
trajectories exhibit strong temporal structure: they can often be predicted from
their recent past, and their increments concentrate in a low-dimensional
temporal subspace. This paper formalizes these phenomena via
\emph{prediction-based path-length} and an SVD-derived \emph{predictable rank},
and shows that both online convex optimization and smooth nonconvex optimization
admit guarantees whose scale is governed by these \emph{measurable} temporal
complexity parameters.

Given gradients $\{g_t\}_{t=0}^T$ and a history-based predictor $\{m_t\}_{t=0}^T$,
we define the prediction-based path-length
$P_T(m)=\sum_{t=0}^T\|g_t-m_t\|^2$ and the normalized predictability index
$\kappa_T(m)=P_T(m)/\sum_{t=0}^T\|g_t\|^2$.
\emph{Calibration:} the trivial predictor $m_t\equiv 0$ yields $\kappa_T(m)=1$ exactly,
so values $\kappa_T(m)\approx 1$ indicate correct-scale tracking, $\kappa_T(m)\ll 1$
indicates near-perfect tracking, and $\kappa_T(m)\gg 1$ indicates unstable or
over-extrapolative prediction. We also form the increment matrix
$H=[g_1-g_0,\dots,g_T-g_{T-1}]$ and define a predictable rank $r^\star(\epsilon)$
as the number of singular directions needed to capture $(1-\epsilon)$ of the
increment energy.

We prove representative results: (i) in online convex optimization, an
optimistic mirror descent bound scales as
$\Regret(T)\lesssim D_\Phi\sqrt{P_T^\star(\cM)}$ for a predictor class $\cM$;
(ii) in smooth nonconvex optimization, for standard first-order updates that use a
history-based \emph{proxy direction} for the current gradient, stationarity bounds
degrade additively by the \emph{average} proxy error; and
(iii) the minimal path-length over rank-$r$ increment predictors equals the
Frobenius residual of the best rank-$r$ approximation of $H$, making
$r^\star(\epsilon)$ an intrinsic temporal dimension parameter.

Empirically, across convolutional networks, vision transformers, small
transformers, MLPs, and GPT-2 (multiple optimizers), simple predictors such as one-step and EMA achieve stable $\kappa_T(m)$ near the zero-predictor baseline ($\kappa=1$), and a few dozen singular
directions explain most increment energy in a $k=256$ random projection despite
parameter counts up to $10^8$. These findings support a \emph{Predictable Gradient
Manifold} view of deep learning optimization: training trajectories are locally
predictable and temporally low-rank, and optimization complexity is often better
parameterized by $(P_T,r^\star)$ than by $(T,d)$.
\end{abstract}


% ============================================================
\section{Introduction}
\label{sec:intro}
% ============================================================

First-order methods (SGD, AdamW, RMSprop, etc.) dominate deep learning.
Classical analyses typically assume worst-case gradient sequences (adversarial
online learning) or high-variance stochasticity, leading to horizon-driven
complexity such as $\Theta(\sqrt{T})$ regret and $O(1/(\eta T))$ stationarity
rates in smooth nonconvex optimization \cite{Hazan2016,GhadimiLan13}.
Yet real training runs are not adversarial: gradients are correlated across
steps, drift smoothly, and often appear to evolve within a low-dimensional
temporal subspace.

This paper formalizes that structure and uses it to define a \emph{measurable}
complexity regime for optimization. The claim is not that deep learning is
intrinsically “easy,” but that its difficulty is frequently governed by
\emph{temporal predictability} and \emph{intrinsic temporal dimension}, rather
than by worst-case horizon and ambient parameter dimension.

\subsection{Predictable Gradient Manifold Hypothesis (local form)}
\label{sec:pgmh}

Let $g_t$ denote the gradient (or a gradient estimate) at step $t$.
A \emph{temporal predictor} is a sequence $m_t$ where $m_t$ depends only on the past,
i.e., it is measurable with respect to $\sigma(\theta_t, g_0,\dots,g_{t-1})$ (no peek at $g_t$). Informally, a training run exhibits a predictable
gradient manifold if, over windows of steps:
(i) prediction errors $\|g_t-m_t\|$ are controlled by simple history-based
predictors; and (ii) increment directions $g_t-g_{t-1}$ concentrate in a
low-dimensional temporal subspace.

We capture these with two measurable objects:
\begin{itemize}[leftmargin=*]
\item \textbf{Prediction-based path-length} $P_T(m)$, measuring how closely a predictor tracks the gradient trajectory.
\item \textbf{Predictable rank} $r^\star(\epsilon)$, measuring the intrinsic temporal dimension of gradient drift.
\end{itemize}

% --- Concept schematic (keep if you want it) ---
\begin{figure}[t]
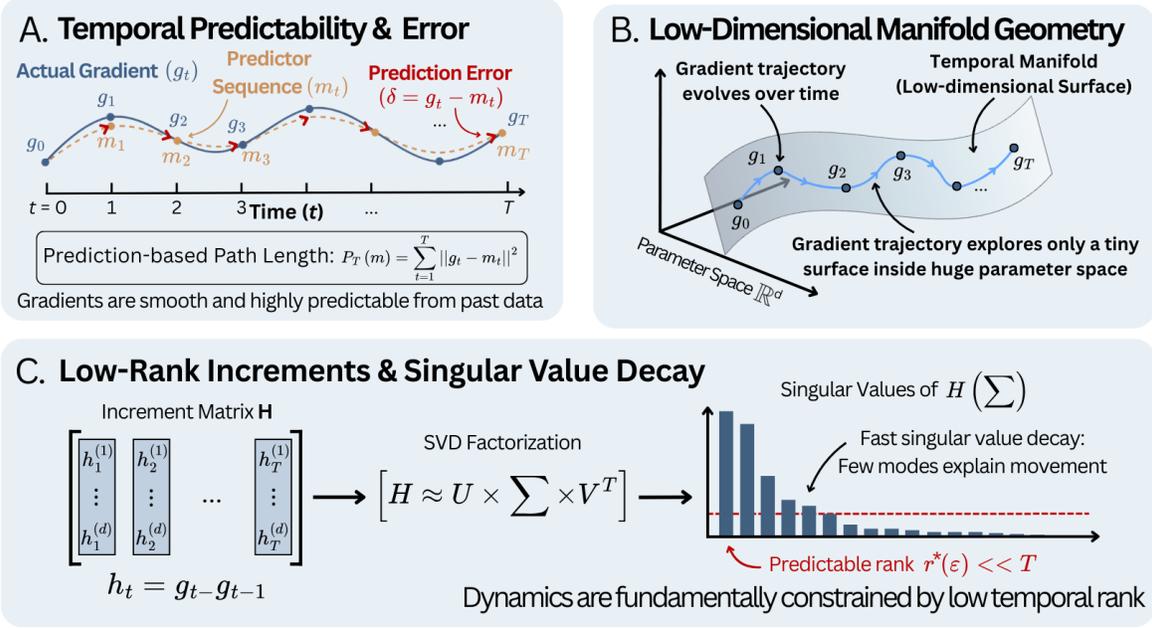

  \centering
  % If you keep your old schematic, put it in figs/ as figure.png (or change name here)
  \Fig{figure.png}{0.95}
  \caption{\textbf{Conceptual overview of predictable gradient manifolds and their associated
  complexity measures.}
  (a) Gradients $\{g_t\}$ evolve over time and are tracked by temporal predictors
  $\{m_t\}$, producing prediction errors $\delta_t = g_t - m_t$ whose squared norms
  accumulate into the prediction-based path-length
  $P_T(m)=\sum_t \|\delta_t\|^2$.
  (b) In the ambient parameter space $\R^d$, the gradient sequence evolves along a thin, low-dimensional \emph{temporal manifold}.
  (c) Gradient increments $h_t = g_t - g_{t-1}$ form an increment matrix
  $H=[h_1,\dots,h_T]$ whose singular values decay rapidly; a small predictable rank
  $r^\star(\epsilon)$ captures most temporal drift energy.}
  \label{fig:pgm}
\end{figure}

% --- NEW: Main empirical figure 1 (kappa over time) ---
\begin{figure}[t]
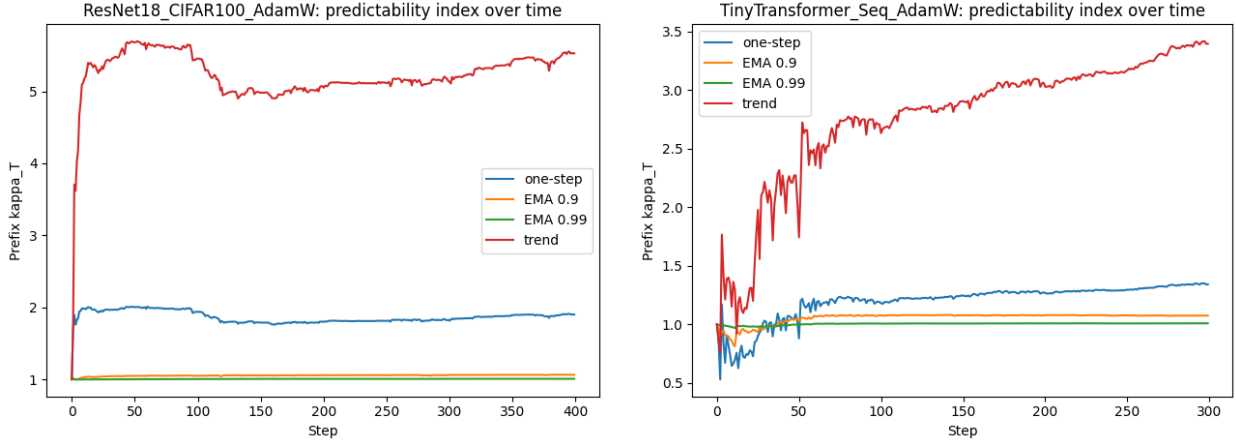

  \centering
  \begin{subfigure}[t]{0.49\linewidth}
    \centering
    % Put this file at: figs/fig_resnet18_kappa.png
    \Fig{fig_resnet18_kappa.png}{1.0}
    \caption{ResNet-18 (CIFAR-100).}
    \label{fig:kappa_time_resnet}
  \end{subfigure}\hfill
  \begin{subfigure}[t]{0.49\linewidth}
    \centering
    % Put this file at: figs/fig_tinytransformer_kappa.png
    \Fig{fig_tinytransformer_kappa.png}{1.0}
    \caption{Tiny Transformer (synthetic sequence).}
    \label{fig:kappa_time_tiny}
  \end{subfigure}
  \caption{\textbf{Predictability is stable over training.}
  Predictor-conditional windowed (or logged-interval) predictability indices
  $\kappa$ remain $O(1)$ for simple history-based predictors (one-step, EMA),
  supporting the \emph{local} Predictable Gradient Manifold hypothesis.}
  \label{fig:kappa_time}
\end{figure}

\subsection{Contributions}
\label{sec:contrib}

\begin{enumerate}[leftmargin=*]
\item We define prediction-based path-length $P_T(m)$, a normalized predictability index $\kappa_T(m)$, and an SVD-based predictable rank $r^\star(\epsilon)$.
\item We give representative convex and nonconvex guarantees whose scale is governed by $P_T(m)$ (or $P_T^\star(\cM)$).
\item We show that the best rank-$r$ increment predictor achieves error equal to the SVD tail energy of the increment matrix.
\item We provide empirical evidence across architectures and optimizers; full protocol and additional diagnostics appear in the appendix.
\end{enumerate}

% ============================================================
\section{Setup and Complexity Measures}
\label{sec:measures}
% ============================================================

We work in $\R^d$ with the Euclidean norm unless stated otherwise.
Let $\{g_t\}_{t=0}^T \subset \R^d$ be a gradient sequence and
$\{m_t\}_{t=0}^T$ be a history-based predictor (i.e., $m_t$ is measurable with
respect to $\sigma(\theta_t, g_0,\dots,g_{t-1})$, so it cannot “peek” at $g_t$).

\subsection{Prediction-based path-length and predictability index}

\begin{definition}[Prediction-based path-length]
\label{def:pathlength}
For a predictor $m$, define
\[
P_T(m)\;:=\;\sum_{t=0}^T \|g_t-m_t\|^2.
\]
For a predictor class $\cM$, define the optimal path-length
$P_T^\star(\cM):=\inf_{m\in\cM}P_T(m)$.
\end{definition}

\begin{definition}[Predictability index]
\label{def:kappa}
Let $G_T:=\sum_{t=0}^T\|g_t\|^2$. If $G_T>0$, define
\[
\kappa_T(m)\;:=\;\frac{P_T(m)}{G_T}.
\]
\end{definition}

\paragraph{Calibration and interpretation (conditional, and why Trend can be large).}
The trivial predictor $m_t\equiv 0$ yields $\kappa_T(m)=1$ exactly, providing a
reference scale for interpreting tables and plots: $\kappa_T(m)\approx 1$ means the
predictor tracks gradients at the correct overall scale, $\kappa_T(m)\ll 1$ indicates
near-perfect tracking, and $\kappa_T(m)\gg 1$ indicates unstable or over-extrapolative
prediction.
More generally, $\kappa_T(m)$ is a dimensionless \emph{relative prediction error}
conditioned on the chosen predictor $m$ (or predictor class $\cM$). Different predictors
yield different $\kappa_T(m)$; in particular, aggressive extrapolations (e.g.\ trend)
can amplify predictor norms, increasing $\kappa_T(m)$ even if dominant directions are
captured. A basic universal bound is in Appendix~\ref{app:kappa}.

\subsection{Increments, SVD, and predictable rank}

Define increments $h_t:=g_t-g_{t-1}$ for $t\ge 1$ and the increment matrix
\[
H:=[h_1,\dots,h_T]\in\R^{d\times T}.
\]

\begin{definition}[Predictable rank]
\label{def:predrank}
Let $H$ have singular values $\sigma_1\ge \sigma_2\ge\cdots\ge 0$.
For $\epsilon\in(0,1)$ define
\[
r^\star(\epsilon)
:=\min\left\{r:\frac{\sum_{i=1}^r\sigma_i^2}{\sum_{i\ge 1}\sigma_i^2}\ge 1-\epsilon\right\}.
\]
\end{definition}

\paragraph{Interpretation.}
$r^\star(\epsilon)$ is the smallest temporal dimension capturing a $(1-\epsilon)$
fraction of increment energy. In many deep learning runs, singular values decay
steeply (often in projected space), suggesting a small intrinsic temporal dimension
over local windows.

% ============================================================
\section{Convex Online Optimization: Regret Scales with Path-Length}
\label{sec:convex}
% ============================================================

We state a representative convex result in the predictable-sequence style
\cite{RakhlinSridharan13Online}. Proof appears in Appendix~\ref{app:convex}.

\subsection{Setting}

Let $\Theta\subset\R^d$ be convex and $f_t:\Theta\to\R$ convex.
At round $t$, the learner plays $\theta_t$, observes $g_t\in\partial f_t(\theta_t)$,
and incurs $f_t(\theta_t)$. Regret is
\[
\Regret(T):=\sum_{t=1}^T f_t(\theta_t)-\min_{\theta\in\Theta}\sum_{t=1}^T f_t(\theta).
\]
Let $\Phi$ be a $1$-strongly convex mirror map with Bregman divergence
$B_\Phi(\cdot,\cdot)$ and diameter
$D_\Phi^2:=\sup_{\theta,\theta'\in\Theta}B_\Phi(\theta,\theta')$.

\subsection{Result}

\begin{theorem}[Path-length regret bound (optimistic mirror descent)]
\label{thm:pmd}
Assume $\|g_t\|_\ast\le G$ and define $\delta_t:=g_t-m_t$. Then for an optimistic
mirror descent update (Appendix~\ref{app:convex}), for any $\eta>0$,
\[
\Regret(T)\;\le\;\frac{D_\Phi^2}{\eta}+\frac{\eta}{2}\sum_{t=1}^T\|\delta_t\|_\ast^2
\;=\;\frac{D_\Phi^2}{\eta}+\frac{\eta}{2}\Big(P_T(m)-\|g_0-m_0\|^2\Big).
\]
Choosing $\eta=D_\Phi/\sqrt{P_T(m)-\|g_0-m_0\|^2}$ yields
$\Regret(T)\le \sqrt{2}\,D_\Phi\sqrt{P_T(m)-\|g_0-m_0\|^2}$.
Moreover, for a predictor class $\cM$,
\[
\Regret(T)\;\le\;\sqrt{2}\,D_\Phi\sqrt{P_T^\star(\cM)}.
\]
\end{theorem}

\paragraph{Takeaway.}
When a simple predictor tracks gradients well (small $P_T(m)$), regret scales with
that \emph{measurable} predictability rather than $\sqrt{T}$.

% ============================================================
\section{Smooth Nonconvex Optimization: Stationarity with Proxy Directions}
\label{sec:nonconvex}
% ============================================================

We give a nonconvex statement showing that using a history-based \emph{proxy direction}
for the current gradient incurs an additive complexity term equal to the average proxy
error. This is best viewed as an \emph{analysis lens} for standard deep learning
training (SGD/momentum/Adam-style updates), rather than as a proposal of a new optimizer.
Proof appears in Appendix~\ref{app:nonconvex}.

\begin{definition}[Gradient descent with history-based proxy directions]
\label{def:proxy_gd}
Let $F:\R^d\to\R$ be differentiable. Let $g_t:=\nabla F(\theta_t)$ denote the true gradient.
Let $m_t$ be any history-based proxy (measurable w.r.t.\ $\sigma(\theta_t, g_0,\dots,g_{t-1})$), and
define the proxy error $\delta_t:=g_t-m_t$. Consider the update
\[
\theta_{t+1}=\theta_t-\eta m_t.
\]
Define $P_{T-1}(m):=\sum_{t=0}^{T-1}\|\delta_t\|^2$.
\end{definition}

\begin{theorem}[Nonconvex convergence with proxy/prediction error]
\label{thm:nonconvex}
Assume $F$ is $L$-smooth and bounded below by $F_\star$, and $\eta\le 1/L$.
Then the iterates of Definition~\ref{def:proxy_gd} satisfy
\[
\frac{1}{T}\sum_{t=0}^{T-1}\|\nabla F(\theta_t)\|^2
\;\le\;
\frac{2(F(\theta_0)-F_\star)}{\eta T}+\frac{P_{T-1}(m)}{T}.
\]
In particular,
\[
\min_{0\le t<T}\|\nabla F(\theta_t)\|^2
\;\le\;
\frac{2(F(\theta_0)-F_\star)}{\eta T}+\frac{P_{T-1}(m)}{T}.
\]
\end{theorem}

\paragraph{Takeaway.}
Smooth nonconvex optimization inherits the usual $O(1/(\eta T))$ term plus an additive
\emph{average proxy error}. When $m_t$ is instantiated as a temporal predictor of $g_t$,
this term is exactly the average prediction error, motivating local/windowed predictors
in regimes where predictability is primarily local in time.

% ============================================================
\section{Low-rank Increments: Intrinsic Temporal Dimension}
\label{sec:lowrank}
% ============================================================

We connect prediction to low-rank structure and justify predictable rank as a
complexity parameter. Proof is short and included here.

\subsection{Rank-$r$ increment predictors}

Consider predictors of the form (for $t\ge 1$)
\[
m_t = g_{t-1}+Uv_t,\qquad U\in\R^{d\times r},\ v_t\in\R^r.
\]
Then $\delta_t=g_t-m_t = (g_t-g_{t-1})-Uv_t=h_t-Uv_t$.

\begin{proposition}[Low-rank residual equals minimal increment prediction error]
\label{prop:equiv}
Let $H=[h_1,\dots,h_T]$. Then
\[
\inf_{U,V:\,\rank(UV)\le r}\sum_{t=1}^T\|h_t-Uv_t\|^2
\;=\;\min_{\rank(M)\le r}\|H-M\|_F^2
\;=\;\sum_{i>r}\sigma_i^2.
\]
Equivalently, the minimal increment-prediction error over rank-$r$ predictors equals
the SVD tail energy of $H$.
\end{proposition}

\begin{proof}
Stacking columns gives $\sum_{t=1}^T\|h_t-Uv_t\|^2=\|H-UV\|_F^2$.
Minimizing over rank-$r$ matrices is the Eckart--Young--Mirsky theorem.
\end{proof}

\paragraph{Implication.}
If $r^\star(\epsilon)$ is small, there exist low-rank increment predictors with small
prediction error, hence small $P_T(m)$ and sharper optimization guarantees.

% ============================================================
\section{Empirical Evidence (Summary)}
\label{sec:experiments}
% ============================================================

We summarize the empirical pattern; the full training protocol, datasets,
hyperparameters, and additional diagnostics appear in Appendix~\ref{app:experiments}.

\paragraph{What we measure.}
We log gradients (or projected gradients) and compute:
(i) $\kappa_T(m)$ for simple predictors (one-step, EMA, trend), and
(ii) predictable ranks $r^\star(\epsilon)$ from the SVD of increment matrices.

\paragraph{Headline observation.}
Across ResNet-18 and ViT-Tiny on CIFAR-100, a small Transformer on synthetic
sequences, a 3-layer MLP on tabular data, and GPT-2 on WikiText-2 (multiple
optimizers), simple predictors yield stable $\kappa_T(m)=O(1)$ and the increment
matrix exhibits steep singular value decay in a $k=256$ random projection.

% --- NEW: Main empirical figure 2 (spectra) ---
\begin{figure}[t]
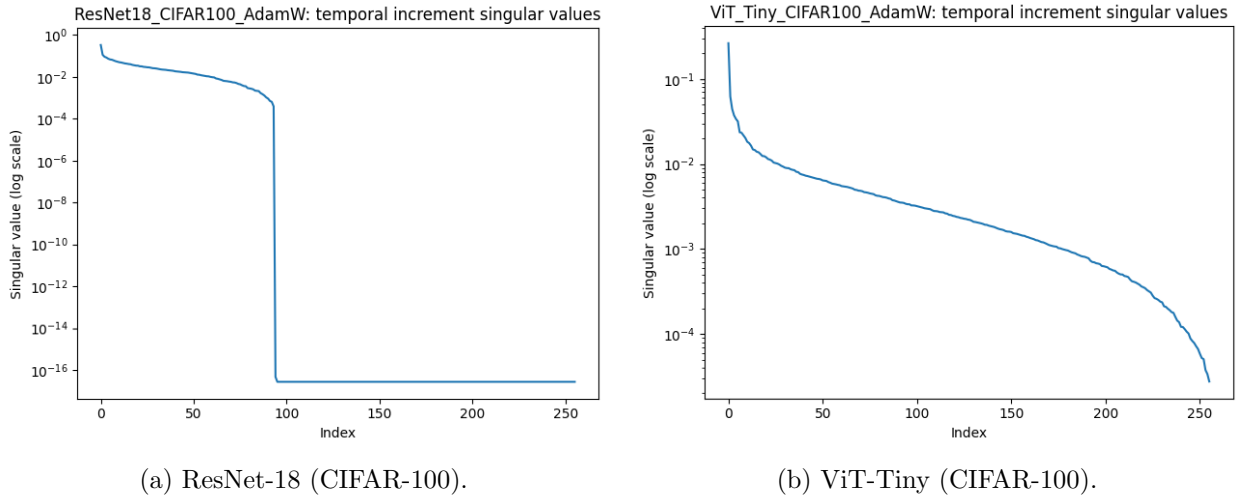

  \centering
  \begin{subfigure}[t]{0.49\linewidth}
    \centering
    % Put this file at: figs/fig_resnet18_spectrum.png
    \Fig{fig_resnet18_spectrum.png}{1.0}
    \caption{ResNet-18 (CIFAR-100).}
    \label{fig:spectrum_resnet}
  \end{subfigure}\hfill
  \begin{subfigure}[t]{0.49\linewidth}
    \centering
    % Put this file at: figs/fig_vittiny_spectrum.png
    \Fig{fig_vittiny_spectrum.png}{1.0}
    \caption{ViT-Tiny (CIFAR-100).}
    \label{fig:spectrum_vit}
  \end{subfigure}
  \caption{\textbf{Increment dynamics are temporally low-rank.}
  Singular values of the increment matrix (computed in a $k=256$ random projection)
  decay rapidly, implying a small predictable rank $r^\star(\epsilon)$ for fixed $\epsilon$.}
  \label{fig:spectrum_main}
\end{figure}

\begin{table}[t]
  \centering
  \caption{Predictability index $\kappa_T(m)$ (projected $k=256$). The zero predictor $m_t\equiv 0$ yields $\kappa=1$ exactly.}
  \label{tab:kappa_main}
  \vspace{0.5em}
  \begin{tabular}{lcccc}
    \toprule
    \textbf{Run} & \textbf{one-step} & \textbf{EMA-0.9} & \textbf{EMA-0.99} & \textbf{Trend} \\
    \midrule
    ResNet18\_CIFAR100\_AdamW          & 1.878 & 1.058 & 1.007 & 5.448 \\
    ResNet18\_CIFAR100\_SGDmom         & 1.932 & 1.061 & 1.006 & 5.463 \\
    ViT\_Tiny\_CIFAR100\_AdamW         & 1.711 & 1.017 & 0.989 & 4.957 \\
    TinyTransformer\_Seq\_AdamW        & 1.340 & 1.074 & 1.008 & 3.395 \\
    TinyTransformer\_Seq\_RMSprop      & 3.157 & 1.099 & 1.009 & 11.171 \\
    MLP\_Tabular\_AdamW                & 1.713 & 0.975 & 0.974 & 5.056 \\
    MLP\_Tabular\_SGDmom               & 1.540 & 1.054 & 1.007 & 4.358 \\
    GPT2\_WikiText2\_AdamW             & 1.984 & 1.050 & 1.000 & 5.927 \\
    \bottomrule
  \end{tabular}
\end{table}

\begin{table}[t]
  \centering
  \caption{Predictable ranks $r^\star(\epsilon)$ (projected $k=256$).}
  \label{tab:rank_main}
  \vspace{0.5em}
  \begin{tabular}{lrrrr}
    \toprule
    \textbf{Run} & \textbf{$r^\star(0.10)$} & \textbf{$r^\star(0.05)$} & \textbf{$r^\star(0.01)$} & \textbf{Params}\\
    \midrule
    ResNet18\_CIFAR100\_AdamW          & 23 & 34 & 53 & 11{,}227{,}812 \\
    ResNet18\_CIFAR100\_SGDmom         & 15 & 25 & 49 & 11{,}227{,}812 \\
    ViT\_Tiny\_CIFAR100\_AdamW         &  6 & 17 & 69 &  5{,}543{,}716 \\
    TinyTransformer\_Seq\_AdamW        &  5 &  8 & 20 &     70{,}210 \\
    TinyTransformer\_Seq\_RMSprop      &  3 &  6 & 18 &     70{,}210 \\
    MLP\_Tabular\_AdamW                & 37 & 52 & 87 &     12{,}610 \\
    MLP\_Tabular\_SGDmom               & 25 & 37 & 74 &     12{,}610 \\
    GPT2\_WikiText2\_AdamW             & 28 & 49 & 93 & 124{,}439{,}808 \\
    \bottomrule
  \end{tabular}
\end{table}

\paragraph{Local vs.\ global.}
Predictability and low-rank structure are best interpreted as \emph{local-in-time}:
over windows, gradients are well-approximated by low-dimensional temporal models even
if the global trajectory bends over long horizons.

% ============================================================
\section{Discussion and Outlook}
\label{sec:discussion}
% ============================================================

Our results suggest that many deep learning training runs live in a regime that is
meaningfully different from the classical worst-case view: gradients are often
\emph{locally predictable} from recent history, and the \emph{drift} in gradients
concentrates into a low-dimensional temporal subspace. The two complexity parameters
introduced here---the prediction-based path-length $P_T(m)$ and the predictable rank
$r^\star(\epsilon)$---make these statements operational: they can be computed from
logs, compared across runs, and used to predict when ``optimization difficulty'' is
likely to increase or decrease.

\paragraph{A measurable complexity regime for training.}
Standard optimization bounds typically scale with $T$ and (implicitly or explicitly)
with the ambient dimension $d$. In contrast, Theorem~\ref{thm:pmd} and
Theorem~\ref{thm:nonconvex} show that if there exists a simple temporal predictor
$m$ with small $P_T(m)$, then regret (in convex online settings) and average
stationarity (in smooth nonconvex settings) scale with this \emph{measured}
prediction error rather than the horizon alone.

\paragraph{Why predictable rank matters (and what it buys you).}
The predictable rank $r^\star(\epsilon)$ provides a complementary lens: rather than
measuring error for a fixed predictor, it quantifies the intrinsic temporal dimension
of gradient drift. Proposition~\ref{prop:equiv} shows an exact connection:
low-rank increment prediction is equivalent to approximating the increment matrix
$H$ by a low-rank matrix, with optimal error equal to the SVD tail energy.

\paragraph{Locality, phases, and ``regime shifts.''}
A key empirical theme is locality: predictability is typically strongest over
windows, not necessarily globally. Spikes in windowed $\kappa$ or increases in
windowed predictable rank may serve as signatures of transitions in training dynamics.

\paragraph{Implications for optimizer design.}
If a run exhibits small $P_T(m)$ for simple predictor families, then prediction-aware
updates should reduce effective optimization complexity. This motivates:
rank-adaptive prediction, window-adaptive prediction, and prediction-aware step sizes.

\paragraph{Limitations and what this does \emph{not} claim.}
Predictability is not guaranteed, and low rank is not universal. Some regimes may
exhibit large $P_T(m)$ and slowly decaying spectra. Metrics depend on what gradients
are logged (full vs.\ projected, raw vs.\ preconditioned, etc.).

\paragraph{Open questions.}
Scaling laws for $(P_T, r^\star)$, structure of the temporal subspace, improved
learned predictors, distribution-shift detection, and algorithmic gains remain open.

% ============================================================
\section{Conclusion}
\label{sec:conclusion}
% ============================================================

This work proposes a reframing of optimization complexity in deep learning: instead
of characterizing difficulty primarily by horizon $T$ and ambient dimension $d$, we
characterize it by \emph{measurable temporal structure}. We introduced
prediction-based path-length $P_T(m)$ and predictable rank $r^\star(\epsilon)$,
proved representative convex and nonconvex guarantees governed by these quantities,
and empirically observed stable predictability indices and steep singular value decay
in increment dynamics across diverse architectures and optimizers.

\paragraph{Acknowledgments.}
No competing financial interests are declared.

% ============================================================
\newpage
\appendix
% ============================================================

% ------------------------------------------------------------
\section{A basic bound on $\kappa_T(m)$ and conditional interpretation}
\label{app:kappa}
% ------------------------------------------------------------

This appendix records a simple universal upper bound and clarifies how $\kappa_T(m)$
depends on the predictor.

\begin{lemma}[A universal bound via predictor magnitude]
\label{lem:kappa_bound}
Let $\alpha:=\sup_{0\le t\le T}\frac{\|m_t\|}{\|g_t\|}$ with the convention that
$\|m_t\|/\|g_t\|=0$ if $g_t=0$. Then
\[
\kappa_T(m)=\frac{\sum_{t=0}^T\|g_t-m_t\|^2}{\sum_{t=0}^T\|g_t\|^2}\;\le\;(1+\alpha)^2.
\]
In particular, if $\|m_t\|\le \|g_t\|$ for all $t$ (i.e.\ $\alpha\le 1$), then
$\kappa_T(m)\le 4$.
\end{lemma}

\begin{proof}
For each $t$, $\|g_t-m_t\|\le \|g_t\|+\|m_t\|\le (1+\alpha)\|g_t\|$.
Square and sum over $t$ and divide by $\sum_t\|g_t\|^2$.
\end{proof}

\begin{remark}[Why $\kappa_T$ is conditional (and why Trend can exceed 4)]
The bound above depends on $\alpha$, which is induced by the predictor choice.
EMA predictors typically satisfy $\|m_t\|\lesssim \|g_t\|$ in stable regimes,
while extrapolative predictors can produce $\|m_t\|\gg \|g_t\|$ on noisy or curved
trajectories, yielding larger $\kappa_T(m)$ even if the predictor captures dominant
directions. Therefore $\kappa_T(m)$ should be interpreted as a \emph{predictor-conditional}
relative error, and meaningful comparisons fix a predictor family $\cM$ or report
multiple predictors side by side (as in Table~\ref{tab:kappa_main}).
\end{remark}

% ------------------------------------------------------------
\section{Convex proof details (Theorem~\ref{thm:pmd})}
\label{app:convex}
% ------------------------------------------------------------

We present a standard optimistic mirror descent analysis; see \cite{Hazan2016,RakhlinSridharan13Online}
for general treatments.

\paragraph{Update (one common optimistic form).}
Let $\Phi$ be $1$-strongly convex w.r.t.\ $\|\cdot\|$ on $\Theta$. Define
\[
\theta_{t+1}=\argmin_{\theta\in\Theta}\left\{\eta\,\ip{m_t}{\theta}+B_\Phi(\theta,\theta_t)\right\}.
\]

\begin{lemma}[Three-point inequality]
\label{lem:three_point}
For any $u\in\Theta$,
\[
\eta\,\ip{m_t}{\theta_t-u}
\le
B_\Phi(u,\theta_t)-B_\Phi(u,\theta_{t+1})-B_\Phi(\theta_{t+1},\theta_t).
\]
\end{lemma}

\begin{proof}
Standard from first-order optimality of $\theta_{t+1}$ and Bregman algebra; see \cite{Hazan2016}.
\end{proof}

\begin{proof}[Proof of Theorem~\ref{thm:pmd}]
By convexity, for any comparator $u\in\Theta$,
\[
f_t(\theta_t)-f_t(u)\le \ip{g_t}{\theta_t-u}=\ip{m_t}{\theta_t-u}+\ip{\delta_t}{\theta_t-u}.
\]
Apply Lemma~\ref{lem:three_point} to bound the $m_t$ term. For the error term,
\[
\ip{\delta_t}{\theta_t-u}\le \|\delta_t\|_\ast\|\theta_t-u\|
\le \frac{\eta}{2}\|\delta_t\|_\ast^2+\frac{1}{2\eta}\|\theta_t-u\|^2.
\]
Strong convexity of $\Phi$ implies $\|\theta_t-u\|^2\le 2B_\Phi(u,\theta_t)$.
Summing over $t$ telescopes $B_\Phi(u,\theta_t)$ and yields
\[
\Regret(T)\le \frac{D_\Phi^2}{\eta}+\frac{\eta}{2}\sum_{t=1}^T\|\delta_t\|_\ast^2
=
\frac{D_\Phi^2}{\eta}+\frac{\eta}{2}\Big(P_T(m)-\|g_0-m_0\|^2\Big).
\]
Optimize $\eta$ to obtain the $\sqrt{P_T}$ form and the predictor-class bound with $P_T^\star(\cM)$.
\end{proof}

% ------------------------------------------------------------
\section{Nonconvex proof details (Theorem~\ref{thm:nonconvex})}
\label{app:nonconvex}
% ------------------------------------------------------------

\begin{lemma}[One-step descent with proxy/prediction error]
\label{lem:one_step_nonconvex}
If $F$ is $L$-smooth and $\eta\le 1/L$, then for $\theta_{t+1}=\theta_t-\eta m_t$
with $g_t=\nabla F(\theta_t)$ and $\delta_t=g_t-m_t$,
\[
F(\theta_{t+1})\le F(\theta_t)-\frac{\eta}{2}\|g_t\|^2+\frac{\eta}{2}\|\delta_t\|^2.
\]
\end{lemma}

\begin{proof}
By $L$-smoothness,
\[
F(\theta_{t+1})\le F(\theta_t)+\ip{g_t}{\theta_{t+1}-\theta_t}+\frac{L}{2}\|\theta_{t+1}-\theta_t\|^2.
\]
Plug $\theta_{t+1}-\theta_t=-\eta m_t=-\eta(g_t-\delta_t)$:
\[
F(\theta_{t+1})\le F(\theta_t)-\eta\|g_t\|^2+\eta\ip{g_t}{\delta_t}
+\frac{L\eta^2}{2}\|g_t-\delta_t\|^2.
\]
Use $\ip{g_t}{\delta_t}\le \frac12\|g_t\|^2+\frac12\|\delta_t\|^2$ and
$\|g_t-\delta_t\|^2\le 2\|g_t\|^2+2\|\delta_t\|^2$ to get
\[
F(\theta_{t+1})\le F(\theta_t)+\left(-\eta+\frac{\eta}{2}+L\eta^2\right)\|g_t\|^2
+\left(\frac{\eta}{2}+L\eta^2\right)\|\delta_t\|^2.
\]
If $\eta\le 1/L$, then $-\eta+\frac{\eta}{2}+L\eta^2\le -\frac{\eta}{2}$ and
$\frac{\eta}{2}+L\eta^2\le \eta$, yielding the stated inequality.
\end{proof}

\begin{proof}[Proof of Theorem~\ref{thm:nonconvex}]
Sum Lemma~\ref{lem:one_step_nonconvex} for $t=0,\dots,T-1$:
\[
F(\theta_T)\le F(\theta_0)-\frac{\eta}{2}\sum_{t=0}^{T-1}\|g_t\|^2+\frac{\eta}{2}\sum_{t=0}^{T-1}\|\delta_t\|^2.
\]
Lower bound $F(\theta_T)\ge F_\star$ and rearrange:
\[
\sum_{t=0}^{T-1}\|g_t\|^2\le \frac{2(F(\theta_0)-F_\star)}{\eta}+\sum_{t=0}^{T-1}\|\delta_t\|^2.
\]
Divide by $T$ to obtain the average bound; the minimum bound follows since
$\min_t a_t\le \frac{1}{T}\sum_t a_t$.
\end{proof}

% ------------------------------------------------------------
\section{Additional remarks on predictors}
\label{app:predictors}
% ------------------------------------------------------------

This appendix records the predictor families used in experiments and clarifies what “history-based” means.

\paragraph{History-based predictors.}
A predictor $m_t$ is history-based if it is measurable with respect to $\sigma(\theta_t, g_0,\dots,g_{t-1})$,
so it can depend on the current iterate and past gradients but does not “peek” at $g_t$.

\paragraph{One-step predictor.}
\[
m_t = g_{t-1}\quad (t\ge 1), \qquad m_0=0.
\]

\paragraph{EMA predictor.}
For $\beta\in(0,1)$,
\[
m_t = \beta m_{t-1} + (1-\beta) g_{t-1}\quad (t\ge 1), \qquad m_0=0.
\]

\paragraph{Trend (linear extrapolation) predictor.}
A two-step extrapolation (used only as a baseline; can amplify noise):
\[
m_t = g_{t-1} + \gamma (g_{t-1}-g_{t-2})\quad (t\ge 2), \qquad m_0=0,\ m_1=g_0,
\]
with fixed $\gamma$ (e.g.\ $\gamma=1$). This predictor is still history-based but may satisfy $\|m_t\|\gg \|g_t\|$.

% ------------------------------------------------------------
\section{Projected logging and SVD: why $k=256$ is meaningful}
\label{app:projection}
% ------------------------------------------------------------

When $d$ is large, we compute metrics on a random projection of gradients.
Let $R\in\R^{k\times d}$ have i.i.d.\ entries $R_{ij}\sim\mathcal{N}(0,1/k)$ and define
$\tilde g_t = R g_t$ and $\tilde m_t = R m_t$.

\paragraph{Computed quantities in projected space.}
We report
\[
\tilde P_T(m)=\sum_{t=0}^T \|\tilde g_t-\tilde m_t\|^2,\qquad
\tilde\kappa_T(m)=\frac{\tilde P_T(m)}{\sum_{t=0}^T\|\tilde g_t\|^2},
\]
and define $\tilde H=[\tilde g_1-\tilde g_0,\dots,\tilde g_T-\tilde g_{T-1}]$ and
$r^\star(\epsilon)$ from the singular values of $\tilde H$.

\paragraph{Remark.}
Random projections approximately preserve norms and inner products for sets of vectors of size polynomial in $d$ (Johnson--Lindenstrauss).
Empirically we observe that the qualitative spectrum shape and rank estimates are stable across seeds and moderate changes in $k$.

% ------------------------------------------------------------
\section{Additional empirical diagnostics}
\label{app:more_empirical}
% ------------------------------------------------------------

\paragraph{Predictability versus a universal bound.}
Figure~\ref{fig:kappa_vs_bound} provides an additional diagnostic relating observed
predictability to a simple magnitude-based upper bound (Appendix~\ref{app:kappa}).
This plot is included as a secondary sanity check and is not needed for the main
claims.

\begin{figure}[t]
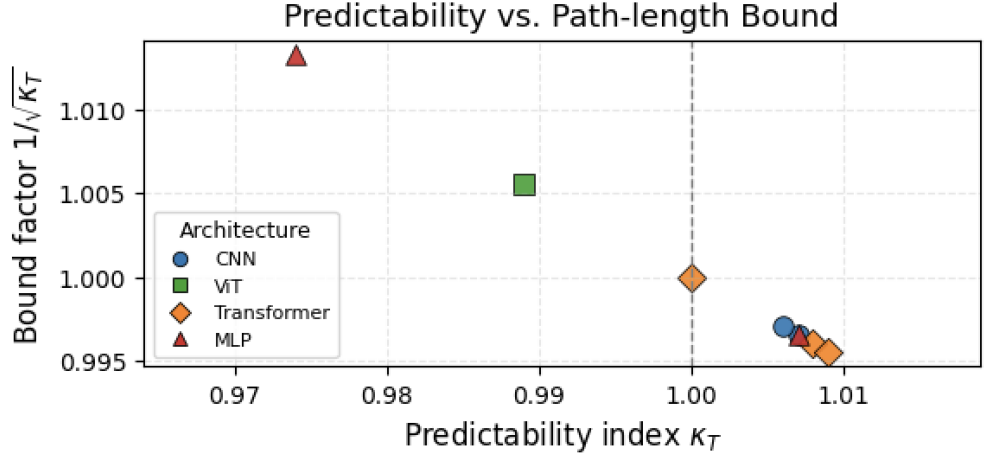

  \centering
  % Put this file at: figs/fig_kappa_vs_bound.png
  \Fig{fig_kappa_vs_bound.png}{0.80}
  \caption{\textbf{Predictability diagnostic.}
  Observed $\kappa$ values compared with a simple universal magnitude-based upper bound.}
  \label{fig:kappa_vs_bound}
\end{figure}

% ------------------------------------------------------------
\section{Experimental details and reproducibility}
\label{app:experiments}
% ------------------------------------------------------------

All experiments reported in the main text are fully reproducible via the
accompanying codebase:

\begin{center}
\url{https://github.com/atbcalvo/predictable-gradient-manifolds} \quad (\texttt{commit: initial})
\end{center}

This appendix records only the information necessary to interpret the reported
metrics; full training scripts, configurations, and logs are provided in the
repository.

\subsection{Logged quantities}

For each training run, we log a sequence of gradient vectors
$\{g_t\}_{t=0}^T$ at fixed intervals during training. When full gradients are
infeasible to store, we log a fixed random projection
$\tilde g_t = R g_t \in \R^k$ with $k=256$, where
$R \sim \mathcal{N}(0, I/k)$ is sampled once per run and held fixed.

All predictability and rank metrics are computed on the logged (projected)
gradients $\tilde g_t$.

\subsection{Predictability metrics}

Given a predictor sequence $\{m_t\}$ (defined in
Appendix~\ref{app:predictors}), we compute the prediction-based path-length
\[
P_T(m) = \sum_{t=0}^T \| \tilde g_t - \tilde m_t \|^2,
\qquad
\kappa_T(m) = \frac{P_T(m)}{\sum_{t=0}^T \|\tilde g_t\|^2}.
\]

Windowed predictability metrics are computed analogously over sliding windows
of fixed length $W$.

\subsection{Predictable rank}

From projected gradients we form increments
$\tilde h_t = \tilde g_t - \tilde g_{t-1}$ and the increment matrix
$\tilde H = [\tilde h_1,\dots,\tilde h_T] \in \R^{k \times T}$.
Predictable rank $r^\star(\epsilon)$ is computed from the singular values of
$\tilde H$ as in Definition~\ref{def:predrank}.

\subsection{Models, datasets, and optimization}

All architectures (ResNet-18, ViT-Tiny, MLP, small Transformer, GPT-2),
datasets (CIFAR-100, WikiText-2, synthetic sequence, tabular), optimizers
(SGD+momentum, AdamW, RMSprop), learning-rate schedules, batch sizes, and
random seeds are specified explicitly in the released configuration files.

Exact parameter counts reported in Table~\ref{tab:rank_main} are produced by
the model definitions in the repository.

% ============================================================
% References
% ============================================================

\end{document}